\documentclass[twocolumn]{article}

\usepackage{arxiv}

\usepackage{microtype}
\usepackage{graphicx}
\usepackage{subfigure}
\usepackage{booktabs} % for professional tables
\usepackage{amsmath}
\usepackage{amsfonts}
\usepackage{bm}  
\usepackage{comment}  
\usepackage{amsthm}
\usepackage{xcolor}

\usepackage{color}
\definecolor{mydarkblue}{rgb}{0,0.08,0.45}
\usepackage{hyperref}
\usepackage{placeins}

\hypersetup{
  linkcolor  = mydarkblue,
  citecolor  = mydarkblue,
  urlcolor   = mydarkblue,
  colorlinks = true,
}

\usepackage{changepage}

\DeclareMathOperator{\rmd}{\mathrm{d}} 
\DeclareMathOperator{\adj}{adj} 
 
\DeclareMathOperator{\Tr}{Tr} 

\newtheorem{lemma}{Lemma}
\newtheorem{theorem}{Theorem}
\newtheorem{corollary}{Corollary}
\newtheorem{remark}{Remark}
\newtheorem{proposition}{Proposition}

\title{Ensemble Kernel Methods, Implicit Regularization and Determinantal Point Processes}

  \author{Joachim Schreurs \\
  Department of Electrical Engineering\\
  ESAT-STADIUS, KU Leuven\\
  Kasteelpark Arenberg 10, B-3001 Leuven, Belgium\\
  \texttt{joachim.schreurs@kuleuven.be}\And  Micha\"el Fanuel \\
  Department of Electrical Engineering\\
  ESAT-STADIUS, KU Leuven\\
  Kasteelpark Arenberg 10, B-3001 Leuven, Belgium\\
  \texttt{michael.fanuel@kuleuven.be}\And Johan A.K. Suykens \\
  Department of Electrical Engineering\\
  ESAT-STADIUS, KU Leuven\\
  Kasteelpark Arenberg 10, B-3001 Leuven, Belgium\\
  \texttt{johan.suykens@kuleuven.be}}

\begin{document}

\vspace{1cm}

\twocolumn[{\centering{\Large \textbf{Ensemble Kernel Methods, Implicit Regularization and Determinantal Point Processes}\par}\vspace{5ex}
	{\large Joachim Schreurs\footnotemark[1], Micha\"el Fanuel\footnotemark[1] and Johan A.K. Suykens\footnotemark[1] 
	}\vspace{6ex}}]\footnotetext[1]{Department of Electrical Engineering, ESAT-STADIUS, KU Leuven,  Kasteelpark Arenberg 10, B-3001 Leuven, Belgium\\
	\scriptsize{\texttt{\{joachim.schreurs;michael.fanuel;johan.suykens@kuleuven.be\}}}}
	
\vspace{2cm}

\begin{adjustwidth}{-5pt}{-5pt}
\begin{abstract}
%It was shown that sampling a more diverse subset results in implicit regularization, which in turn improves the performance of different kernel applications.
By using the framework of Determinantal Point Processes (DPPs), some theoretical results concerning the interplay between diversity and regularization can be obtained. In this paper we show that sampling subsets with kDPPs results in implicit regularization in the context of ridgeless Kernel Regression. Furthermore, we leverage the common setup of state-of-the-art DPP algorithms to sample multiple small subsets and use them in an ensemble of ridgeless regressions. Our first empirical results indicate that ensemble of ridgeless regressors can be interesting to use for datasets including redundant information. 
\end{abstract}
\end{adjustwidth}

\vspace{0.5cm}

\section{Introduction}
\label{sec:Introduction}

Recent work has shown numerous insightful connections between Determinantal Point Processes (DPPs) and implicit regularization which lead to new guarantees and improved algorithms. The so-called DPPs are probabilistic models of repulsion inspired from physics, which are capable of sampling diverse subsets. An extensive overview of the use of DPPs in randomized linear algebra can be found in~\cite{derezinski2020determinantal}. By utilizing DPPs, exact expressions for implicit regularization as well as connections to the double descent curve~\cite{belkin2019reconciling} were derived in~\cite{fanuel2020diversity,derezinski2019exact,CSSP}.
The nice theoretical properties of DPPs sparked the search for efficient sampling algorithms. This resulted in many alternative algorithms for DPPs to reduce the computational cost of preprocessing and/or sampling, including many approximate and heuristic approaches. Some examples are the  exact sampler without eigendecomposition of~\cite{desolneux2018exact,Poulson}, coreDPP of \cite{li2015efficient} or the DPP-VFX algorithm of \cite{derezinski2019exact}. The computational cost is often split in two parts: an initial preprocessing cost and subsequent sampling cost. The latter is typically smaller, which makes the previously mentioned algorithms especially useful for sampling multiple small subsets from a large dataset. 

\begin{figure*}[h]
	\centering
    \begin{minipage}{.248\textwidth}
        \centering
        \includegraphics[width=1\linewidth]{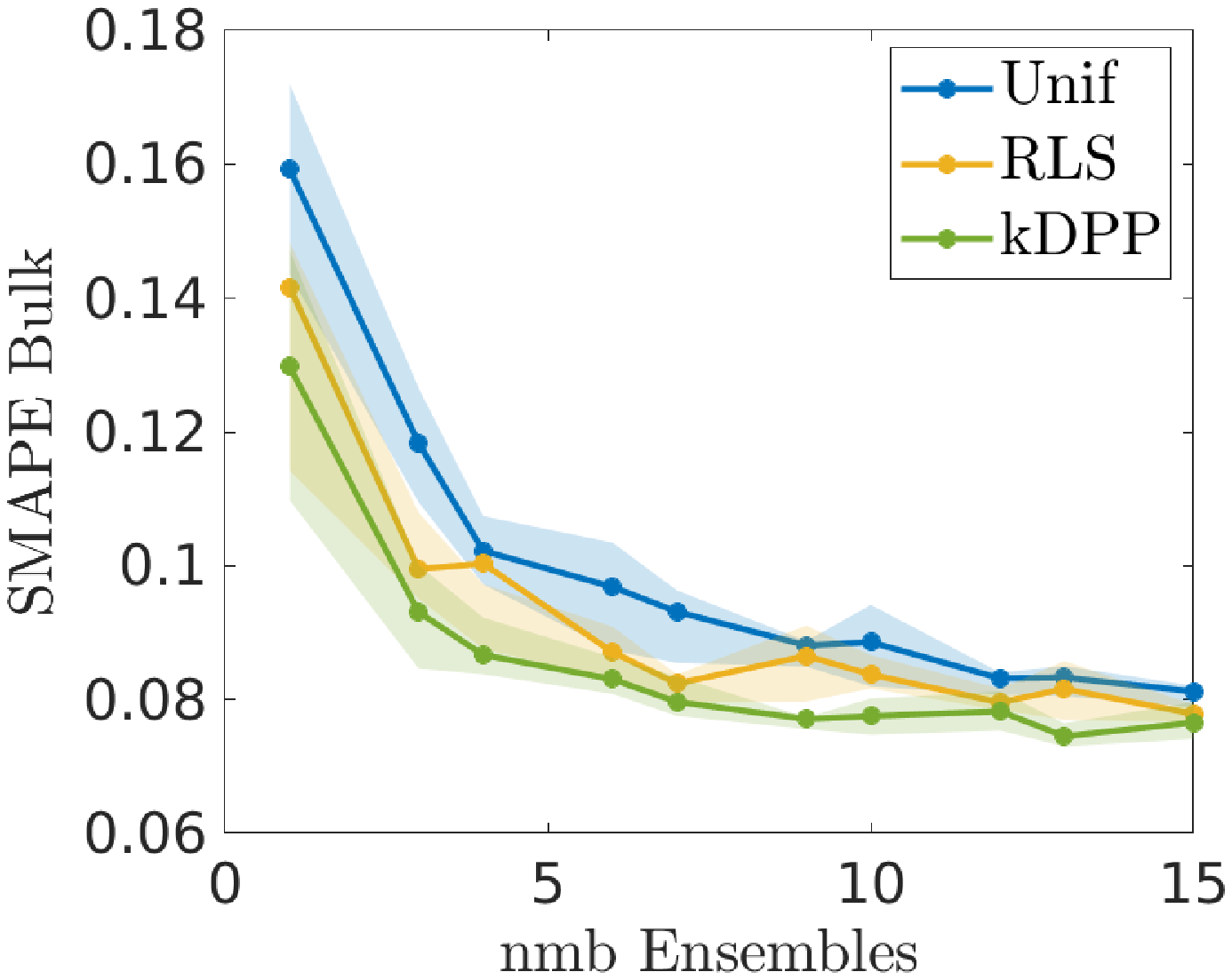}
    \end{minipage}%
    \begin{minipage}{0.248\textwidth}
        \centering
        \includegraphics[width=1\linewidth]{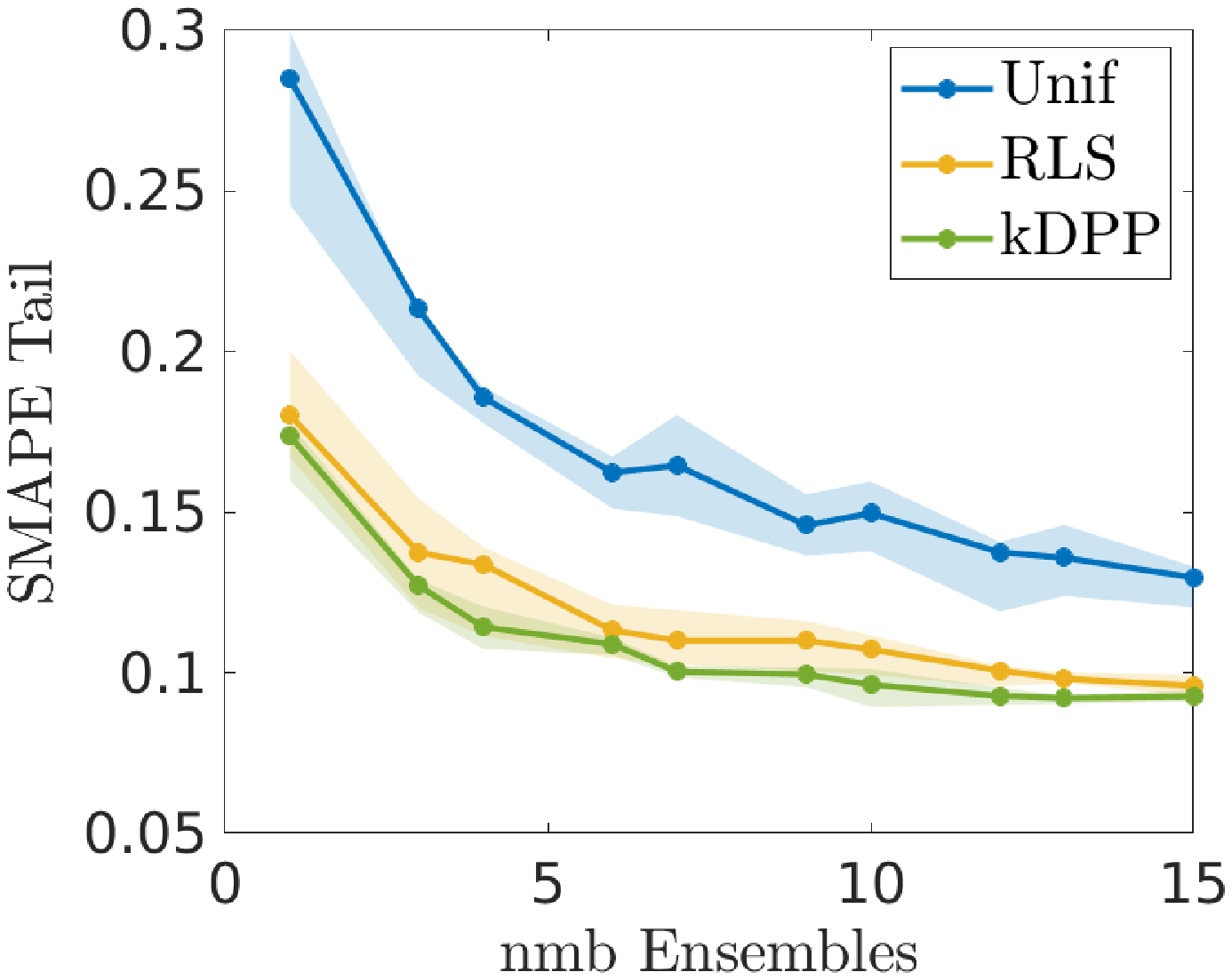}
    \end{minipage}
        \begin{minipage}{.248\textwidth}
        \centering
        \includegraphics[width=1\linewidth]{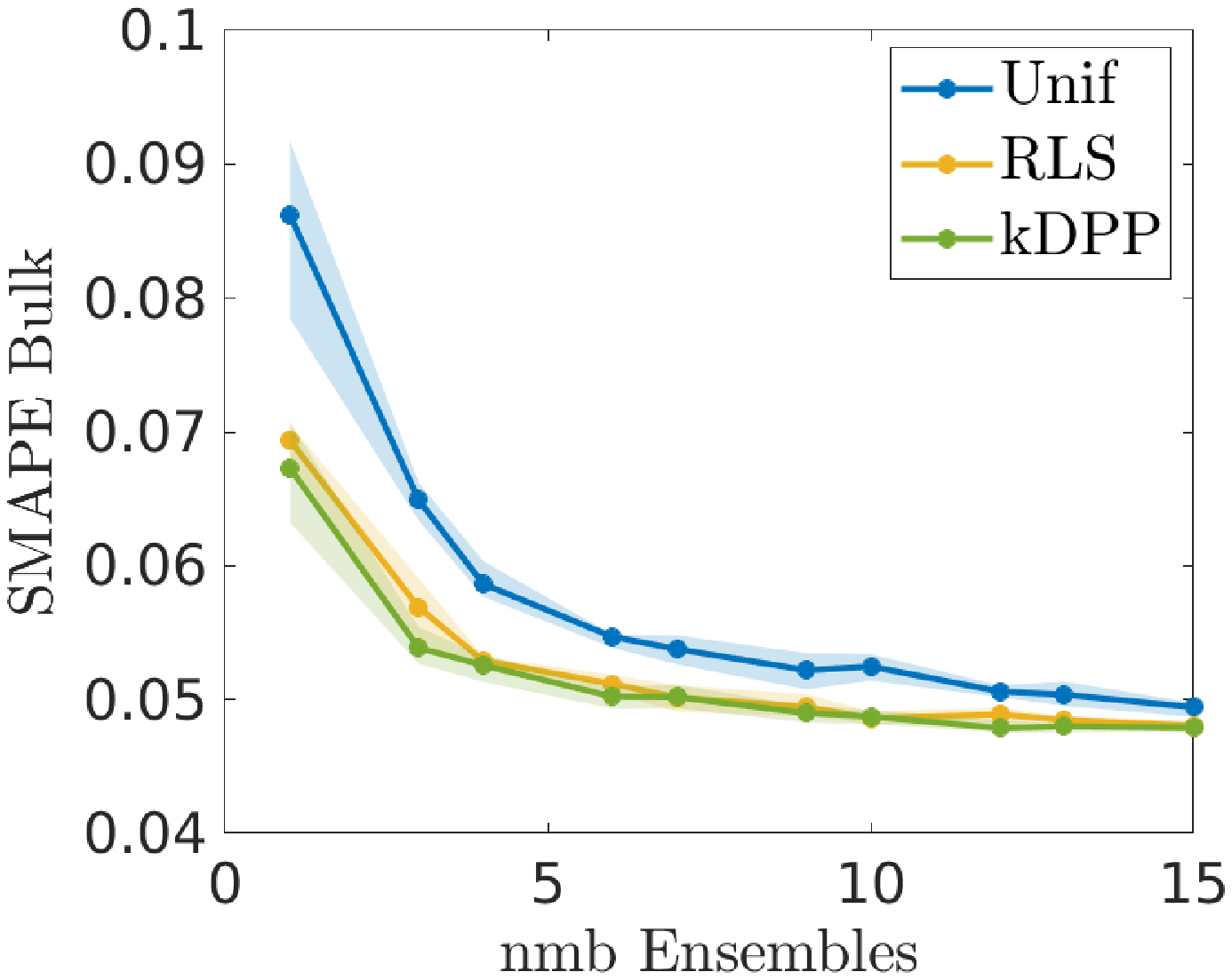}
    \end{minipage}%
    \begin{minipage}{0.248\textwidth}
        \centering
        \includegraphics[width=1\linewidth]{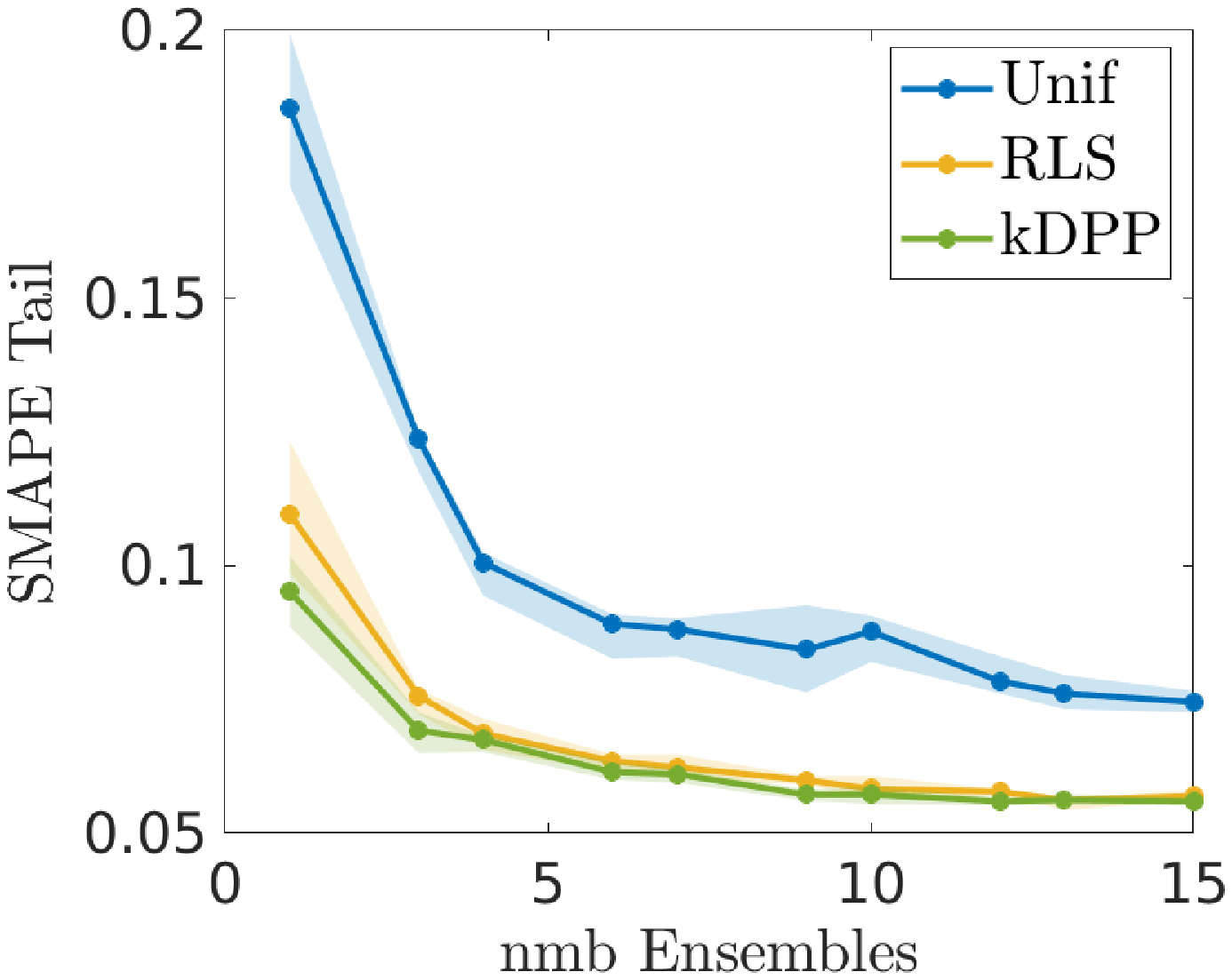}
    \end{minipage}
    \caption{Ensemble KRR on the Abalone and Wine Quality dataset (from left to right).The SMAPE on the bulk and tail of the dataset is given in function of the number of ensembles.}
	\label{fig:KKR:KKR}
\end{figure*}

We extend the work of~\cite{fanuel2020diversity}, where the role of \emph{diversity} within kernel methods was investigated. Namely, a more diverse subset results in implicit regularization, which in turn improves the performance of different kernel applications% like the Nystr\"om approximation, kernel ridge regression, kernel PCA, etc. 
More specifically we generalize the implicit regularization of DPPs to kDPPs, which are DPPs conditioned on a fixed subset size $k$~\cite{kulesza2011k}. Furthermore, we leverage the common setup of state of the art DPP sampling algorithms, to sample multiple small subsets and use them in an ensemble approach. One can loosely characterize these ensemble approaches as methods wherein the data points are divided into smaller subsets, and estimators are trained on the divisions. Their use has shown to improve performance in Nystr\"om approximation~\cite{kumar2009ensemble} and kernel ridge regression~\cite{zhang2013divide,hsieh2014divide,lin2017distributed}. Experiments show a reduction in  error when combining multiple diverse subsets.
\paragraph{Nystr\"om Approximation.}
Let $k(x,y)>0$ be a continuous and \emph{strictly} positive definite kernel. Examples are the Gaussian kernel $k(x,y)= \exp(-\|x-y\|_2^2/2\sigma^2)$ or Laplace Kernel $k(x,y)= \exp(-\|x-y\|_2/\sigma)$. Given data $\{x_i\in \mathbb{R}^d \}_{i\in [n]}$, kernel methods rely on the entries of the Gram matrix $K = [k(x_i,x_j)]_{i,j}$. By assumption, this Gram matrix is invertible.
However, to avoid inverting the full Gram matrix, one often samples a subset of landmarks $\mathcal{C}\subseteq [n]$ with $n\times |\mathcal{C}|$ a sampling matrix $C$ obtained by selecting the columns of the identity matrix indexed by $\mathcal{C}$. Next we define:  $K_{\mathcal{C}} = KC$ and $K_{\mathcal{C}\mathcal{C}} = C^\top K C$. Then, the $n\times n$ kernel matrix $K$ is approximated by a low rank Nystr\"om approximation
$L(K,\mathcal{C}) = K_{\mathcal{C}}K_{\mathcal{C}\mathcal{C}}^{-1} K_{\mathcal{C}}^\top,$
which involves inverting the smaller $K_{\mathcal{C}\mathcal{C}}$.

\paragraph{Ridgeless Kernel Regression.}
Given input-output pairs $\{(x_i,y_i)\in \mathbb{R}^d \times \mathbb{R}\}_{i\in [n]}$, we propose to solve
\begin{equation}
\label{eq:ridgeRegression}
f^\star_{\mathcal{C}}=  \arg\min_{f\in \mathcal{H}}\|f\|_{\mathcal{H}}^2, \text{ s.t. } y_i=f(x_i) \text{ for all } i\in \mathcal{C},
\end{equation}
where $\mathcal{C}\subseteq [n]$ is sampled by using a DPP. Here, $\mathcal{H}$ is the reproducing kernel Hibert space associated with $k$. The expression of the solution is
$
f^\star_{\mathcal{C}}(x) = \bm{k}_x^\top C K_{\mathcal{C}\mathcal{C}}^{-1} C^\top \bm{y},
$
where $\bm{k}_x = [k(x,x_1), \dots, k(x,x_n)]^\top$.
This approximation assumes that some data points can be omitted, contrary to Nystr\"om approximation to Kernel Ridge Regression (KRR) which uses all data points.
We show in this paper that averaging ridgeless regressors yield the solution of a regularized Kernel Ridge Regression calculated over the complete dataset.
For $\mathcal{C}\sim DPP(K/\alpha)$, the expectation of the rigdeless predictors (cfr. Theorem \ref{thm:implicit}) gives the function
\begin{equation} 
\label{eq:expKRR}
\mathbb{E}_{C}[f^\star_{\mathcal{C}} (x)]  = \bm{k}_x^\top (K+\alpha\mathbb{I})^{-1} \bm{y} =: f^\star(x) 
\end{equation} 
which is the solution of Kernel Ridge Regression with a ridge parameter associated to $\alpha$, namely
\[
f^\star =   \arg\min_{f\in \mathcal{H}} \sum_{i=1}^{n} (y_i-f(x_i))^2 + \alpha \|f\|_{\mathcal{H}}^2.
\]
Typically, a large $\alpha>0$ yields a small expected subset size for $DPP(K/\alpha)$.
In light of the expectation result of \eqref{eq:expKRR}, we propose to sample multiple subsets using a DPP and average the ridgeless predictors in an ensemble approach: $\bar{f}= \frac{1}{m} \sum_{i=1}^{m} f^\star_{\mathcal{C}_i}$ with $m$ the number of ensembles.

\paragraph{Determinantal Point Processes}
A more extensive overview of DPPs is given in~\cite{KuleszaT12}.
Let $L$ be a $n\times n$ positive definite symmetric matrix, called L-ensemble. The probability of sampling a subset $\mathcal{C}\subseteq [n]$ is defined as follows
$
\Pr(Y = \mathcal{C}) = \det(L_{\mathcal{C}\mathcal{C}})/\det(\mathbb{I}+L).
$
Where we define $L = K/\alpha$ with $\alpha>0$ and denote the associated process  $DPP_L(K/\alpha)$.
The inclusion probabilities are given by
$
\Pr( \mathcal{C}\subseteq Y) = \det(P_{\mathcal{C}\mathcal{C}}),
$
where
$
P = K(K+\alpha\mathbb{I})^{-1},\label{eq:P}
$ is the marginal kernel associated to the $L$-ensemble $L = K/\alpha$.
The diagonal of this soft projector matrix $P$ gives the Ridge Leverage Scores (RLS) of the data points:
$\bm{\ell}_i = P_{ii} \text{ for } i\in [n],$ which have been used to sample landmarks points in various works~\cite{Bach2013,ElAlaouiMahoney,MuscoMusco} in the context of Nystr\"om approximations. The RLS can be viewed as the importance or uniqueness of a data point. Connections between RLS, DPPs and Christoffel functions were explored in \cite{fanuel2019nystr}. Note that guarantees for DPP sampling for coresets have been derived in \cite{DPPCoresets}.

%\begin{figure*}[h]
%	\centering
%	  \begin{minipage}{.35\textwidth}
%       \centering
%        \includegraphics[width=1\linewidth]{Figures/Parkinson_Perf_Nys.eps}
%    \end{minipage}%
%    \begin{minipage}{0.35\textwidth}
%        \centering
%        \includegraphics[width=1\linewidth]{Figures/adult_Perf_Nys_LS.eps}
%    \end{minipage}
%	\caption{Ensemble Nystr\"om approximation on the Parkinson and Adult dataset.}
%	\label{fig:Nystrom}
%\end{figure*}

\section{Main results}
\subsection{DPP and implicit regularization}
Theorem \ref{thm:implicit} can be found in \cite{fanuel2020diversity} and \cite{DerezinskiAISTATS} in the context of kernel methods and stochastic optimization respectively. It relates the average of pseudo-inverse of kernel submatrices to a regularization inverse of the full kernel matrix.
\begin{theorem}[Implicit regularization]
Let $\mathcal{C}$ be a subset sampled according to $DPP(K/\alpha)$ with $K\succ 0$. Then, we have the identity
$
\mathbb{E}_{C}[C K_{\mathcal{C}\mathcal{C}}^{-1} C^\top] = (K+\alpha\mathbb{I})^{-1}.
$
\label{thm:implicit}
\end{theorem}
 Interestingly, a large regularization parmeter $\alpha>0$ corresponds to small expected subset size $\mathbb{E}[|\mathcal{C}|] = \Tr\left(K(K+\alpha\mathbb{I})
^{-1}\right)$. We now discuss an analogous result in the case of kDPPs, for which the implicit regularization effect can be observed.
\subsection{Analogous result for kDPP sampling}
The elementary symmetric polynomial $e_k(K)$ is proportional to the $(n-k)$-th coefficients of the characteristic polynomial
$
\det(t\mathbb{I}-K) = \sum_{k=0}^{n}(-1)^{k}e_{k}(K)t^{n-k}.
$
Those polynomials are defined on the vector $\bm{\lambda}$ of eigenvalues of $K$.
There are explicitly given by the formula $e_k(\bm{\lambda})=\sum_{1\leq i_1<\dots <i_k\leq n}\lambda_{i_1}\dots\lambda_{i_k}.$
The kDPPs(K) are defined by the subset probabilities
$
\Pr(Y = \mathcal{C}) = \det(K_{\mathcal{C}\mathcal{C}})/e_k(K),
$
and corresponds to DPPs conditioned to a fixed subset size $k$. Now, we state a result analogous to Theorem \ref{thm:implicit}.
\begin{lemma}
Let $\mathcal{C}\sim kDPP(K)$ and $\bm{u},\bm{w}\in\mathbb{R}^n$. We have the identities
\begin{align*}
    &\mathbb{E}_{C}[\bm{u}^\top C K_{\mathcal{C}\mathcal{C}}^{-1}C^\top \bm{w}] = \frac{e_k(K)-e_k(K-\bm{w}\bm{u}^\top)}{e_{k}(K)}\\
    &=\frac{(-1)^{k+1}}{(n-k)!}\frac{\rmd^{(n-k)}}{\rmd t^{n-k}}\left[\frac{\bm{u}^\top \adj(t\mathbb{I}-K)\bm{w}}{e_k(K)}\right]_{t=0},
\end{align*}\label{lem:Expk}
where $\adj$ is the adjugate of a matrix.
\end{lemma}
\begin{proof}
Firstly, we use the matrix determinant lemma:
\[
\bm{u}^\top C K_{\mathcal{C}\mathcal{C}}^{-1}C^\top \bm{w} = \frac{\det(K_{\mathcal{C}\mathcal{C}})-\det(K_{\mathcal{C}\mathcal{C}}-C^\top \bm{w}\bm{u}^\top C)}{\det(K_{\mathcal{C}\mathcal{C}})}.
\]
By taking the expectation over $\mathcal{C}\sim \text{kDPP}(K)$, we find
\[
\mathbb{E}[\bm{u}^\top C K_{\mathcal{C}\mathcal{C}}^{-1}C^\top \bm{w}] = \frac{e_k(K)-e_k(K-\bm{w} \bm{u}^\top)}{e_k(K)} =: \mathcal{E}.
\]
where we used that $\sum_{|\mathcal{C}| =k} \det A_{\mathcal{C}\mathcal{C}} = e_k(A)$ for any square matrix $A$. Next, we use the identity $e_k(K) = \frac{(-1)^k}{(n-k)!}\frac{\rmd^{(n-k)}}{\rmd t^{n-k}} [\det(t\mathbb{I}-K)]_{t=0}$ to obtain the corresponding coefficient of the polynomial $
\det(t\mathbb{I}-K) = \sum_{k=0}^{n}(-1)^{k}e_{k}(K)t^{n-k}
$. Then, we use once more the matrix determinant lemma with the matrix $(t\mathbb{I} -K)$ this time. This gives
\begin{align*}
    \mathcal{E} =\frac{(-1)^{k-1}}{(n-k)!}\frac{\rmd^{(n-k)}}{\rmd t^{n-k}}\left[\frac{\bm{u}^\top (t\mathbb{I}-K)^{-1}\bm{w}}{e_k(K)}\det(t\mathbb{I} -K)\right]_{t=0}.
\end{align*}
Finally, we recall that $\adj(A) = \det(A)A^{-1}$, which completes the proof.
\end{proof}
The implicit regularization due to the diverse sampling is not explicit in Lemma \ref{lem:Expk}. In order to clarify this formula, we write first an equivalent expression for it.
Let the eigendecomposition of $K$ be $K =\sum_{\ell=1}^n \lambda_\ell \bm{v}_\ell \bm{v}_\ell^\top$. Denote by $\bm{\lambda}\in\mathbb{R}^n$ the vector containing the eigenvalues of $K$, sorted such that $\lambda_1\geq \dots\geq \lambda_n$. Let $\bm{\lambda}_{\hat{k}}\in\mathbb{R}^{n-1}$ be the same vector with $\lambda_k$ missing. 
\begin{corollary}
Let $\mathcal{C}\sim kDPP(K)$. We have the identity:
\begin{equation}
\mathbb{E}_{C}[C K_{\mathcal{C}\mathcal{C}}^{-1}C^\top] = \sum_{\ell=1}^{n} \frac{\bm{v}_\ell \bm{v}_\ell^\top}{\lambda_\ell +\frac{e_{k}(\bm{\lambda}_{\hat{\ell}})}{e_{k-1}(\bm{\lambda}_{\hat{\ell}})}}.\label{eq:kDPPreg}
\end{equation}\label{Corollary:spectral}
\end{corollary}
\begin{proof}
To begin with, we expand the adjugate in Lemma \ref{lem:Expk} in the basis of eigenvectors of $K$. This gives
\[
    \adj(t\mathbb{I}-K) = \sum_{\ell =1}^{n} \frac{\prod_{\ell'=1}^n (t-\lambda_{\ell'})}{t-\lambda_{\ell}} \bm{v}_\ell \bm{v}_\ell^\top
%    &=\sum_{\ell =1}^{n} \prod_{\ell'\neq \ell} (t-\lambda_{\ell'}) \bm{v}_\ell \bm{v}_\ell^\top.
\]
Then, by the definition of the polynomials $e_k$ and by noting that $n-k = n-1-(k-1)$, we find
\[
\frac{(-1)^{k-1}}{(n-k)!}\frac{\rmd^{(n-k)}}{\rmd t^{n-k}}\left[\prod_{\ell'\neq \ell} (t-\lambda_{\ell'})\right]_{t=0} = e_{k-1}(\bm{\lambda}_{\hat{\ell}}),
\]
where $\bm{\lambda}_{\hat{\ell}}\in\mathbb{R}^{n-1}$ is the vector $\bm{\lambda}\in\mathbb{R}^{n}$ with $\lambda_{\ell}$ missing. This yields
$
\mathbb{E}_{C}[C K_{\mathcal{C}\mathcal{C}}^{-1}C^\top] = \sum_{\ell =1}^{n} \frac{e_{k-1}(\bm{\lambda}_{\hat{\ell}})}{e_k(\bm{\lambda})} \bm{v}_\ell \bm{v}_\ell^\top.
$
The final identity is obtained by using the following recurrence relation $e_k(\bm{\lambda}) = \lambda_{\ell} e_{k-1}(\bm{\lambda}_{\hat{\ell}}) +  e_k(\bm{\lambda}_{\hat{\ell}})$.
\end{proof}

\begin{figure*}[h]
	\centering
    \begin{minipage}{.248\textwidth}
        \centering
        \includegraphics[width=1\linewidth]{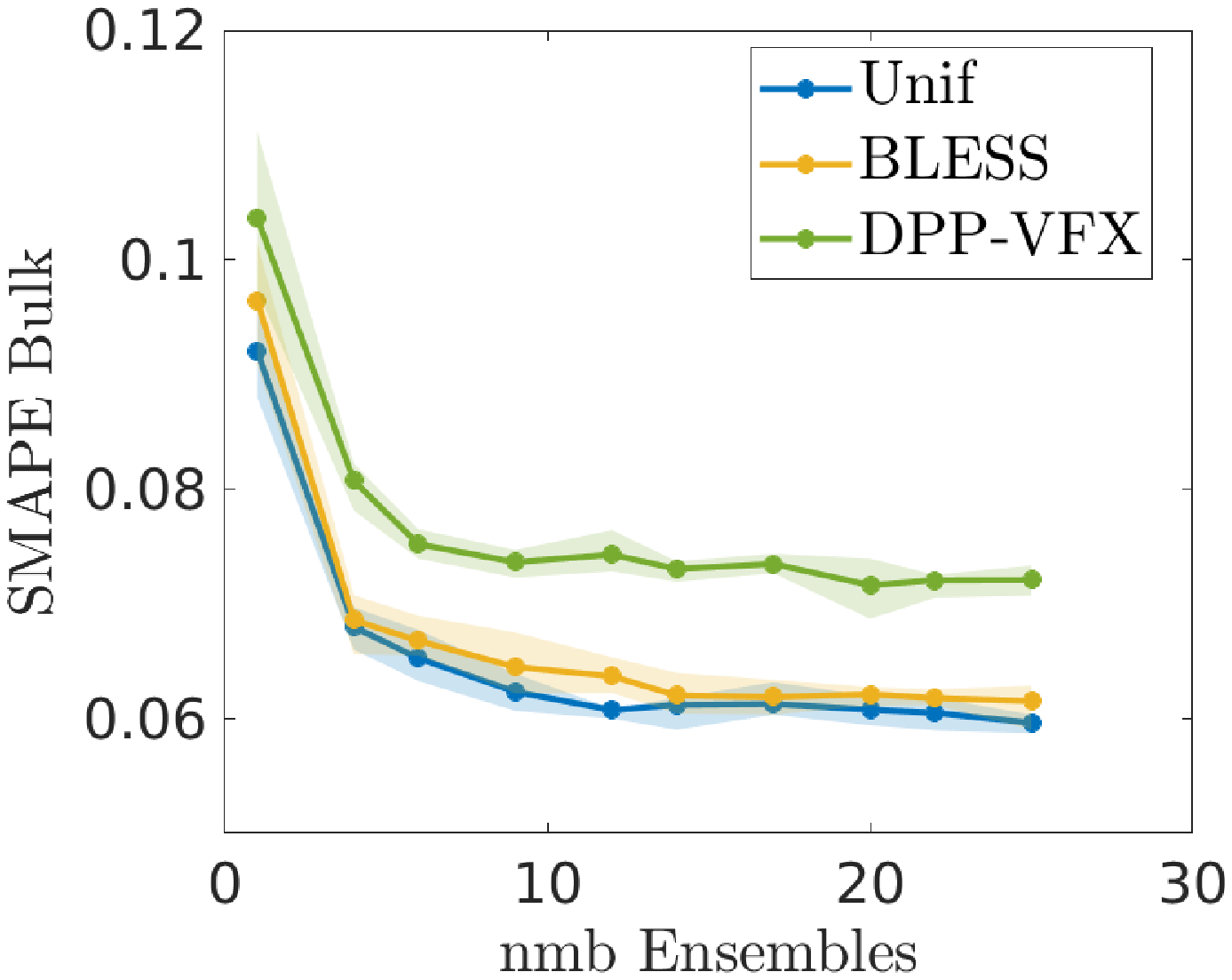}
    \end{minipage}%
    \begin{minipage}{0.248\textwidth}
        \centering
        \includegraphics[width=1\linewidth]{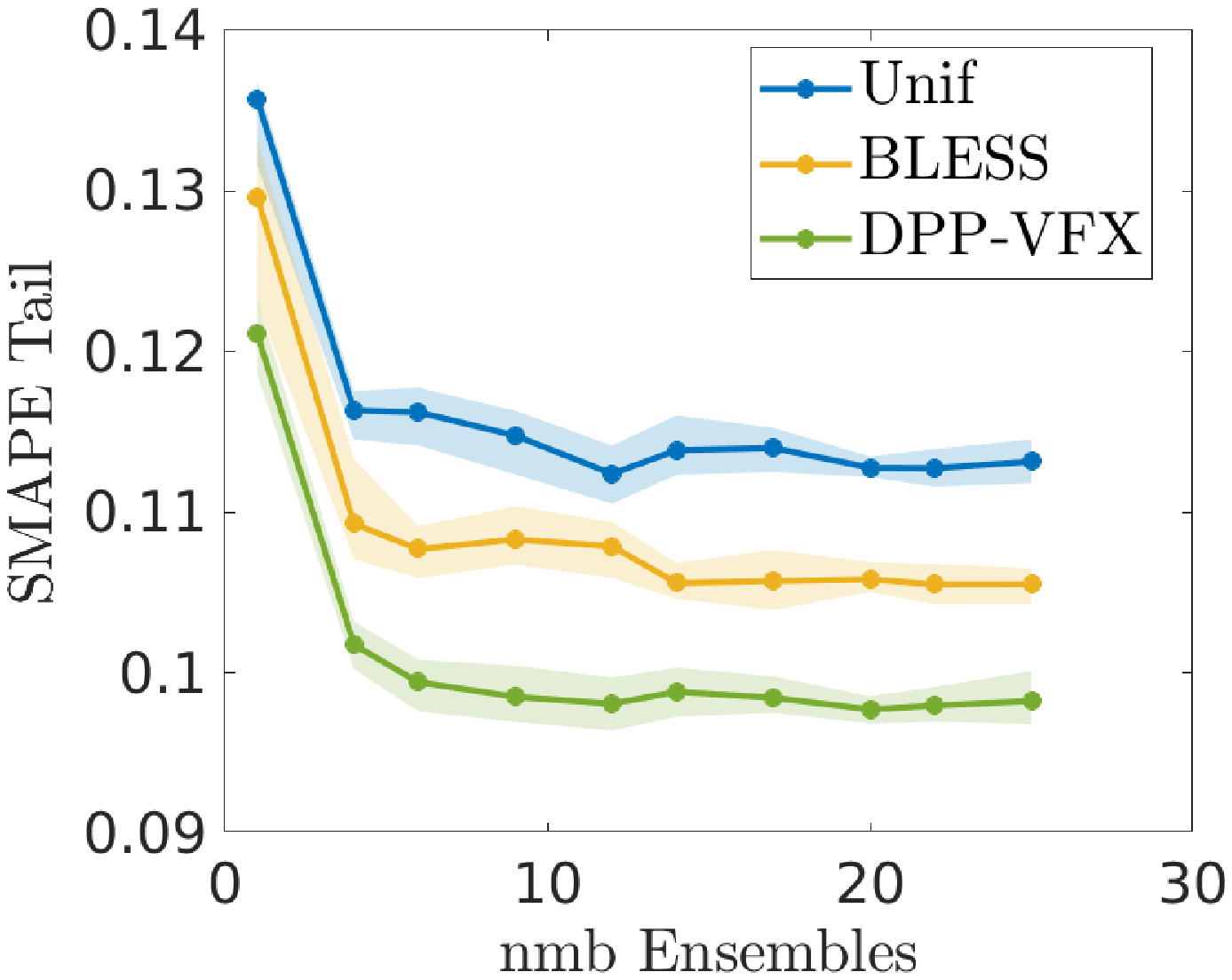}
    \end{minipage}
        \begin{minipage}{.248\textwidth}
        \centering
        \includegraphics[width=1\linewidth]{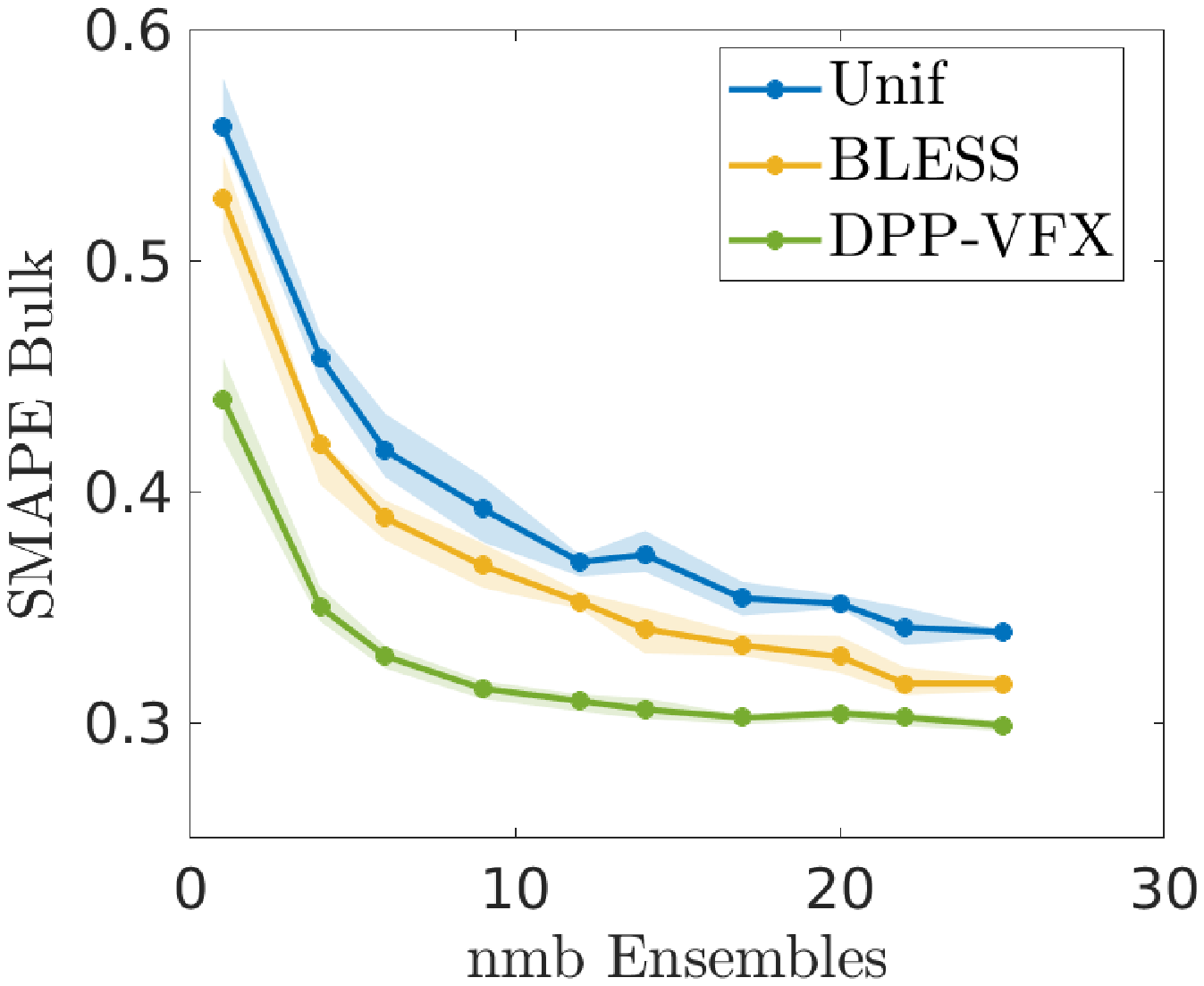}
    \end{minipage}%
    \begin{minipage}{0.248\textwidth}
        \centering
        \includegraphics[width=1\linewidth]{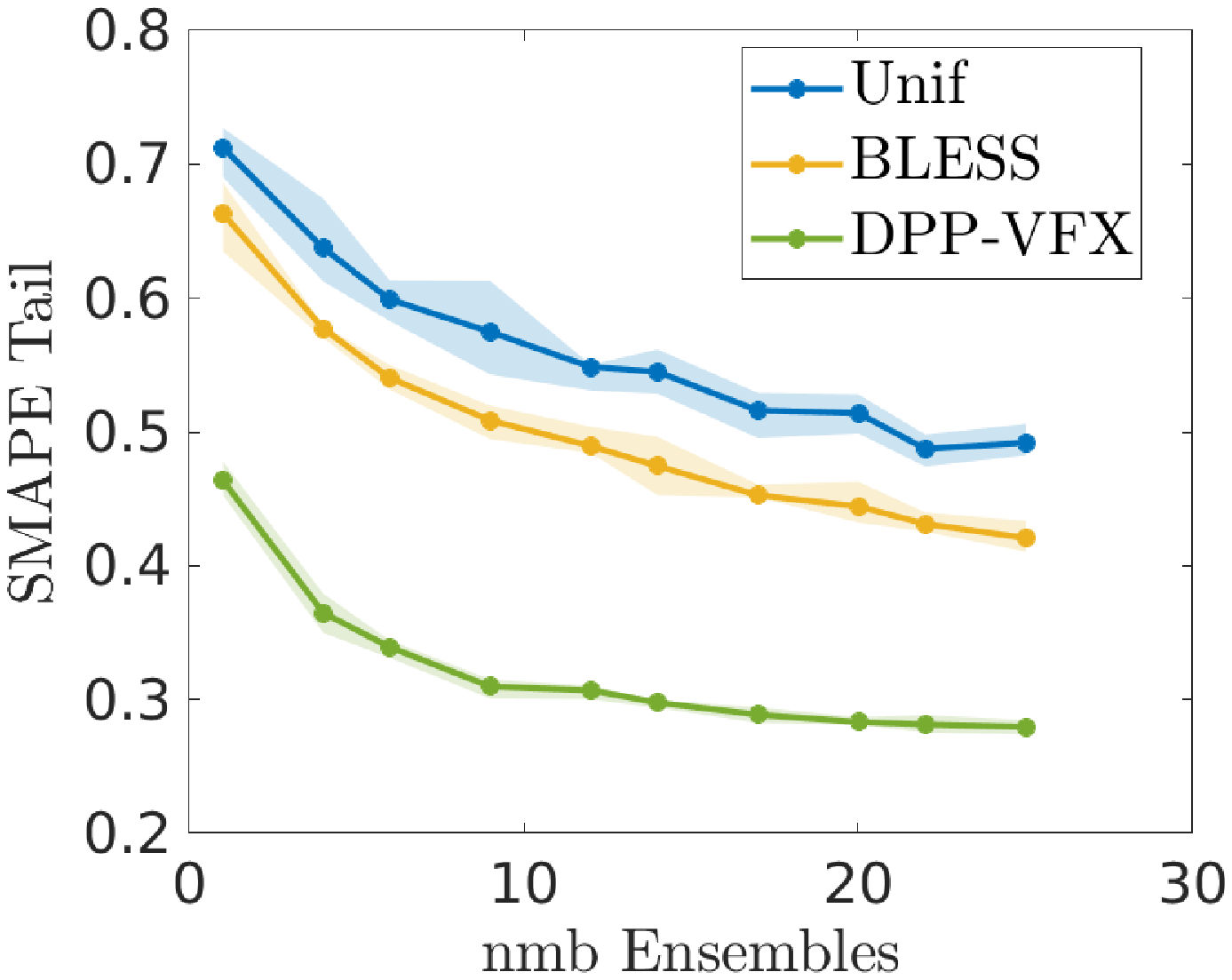}
    \end{minipage}
    \caption{Ensemble KRR on the Bikesharing and CASP dataset (from left to right). The SMAPE on the bulk and tail of the dataset is given in function of the number of ensembles.}
	\label{fig:KKR:KKR_LS}
\end{figure*}
It is now possible to illustrate the connection between Corollary \ref{Corollary:spectral} and implicit regularization. We give a lower bound for the identity in Corollary \ref{Corollary:spectral}. %which shows that the expectation on the lhs of \eqref{eq:kDPPreg} is not unbounded if the smallest eigenvalues of $K$ tend to zero.
\begin{proposition} With the notations defined above, we have
\label{prop:UpperBound}
\begin{equation}
\mathbb{E}_{C}[C K_{\mathcal{C}\mathcal{C}}^{-1}C^\top] \succeq \sum_{\ell=1}^{n} \frac{ \bm{v}_\ell \bm{v}_\ell^\top}{  \lambda_\ell +\alpha},\label{eq:UpperBound}
\end{equation}
where $\alpha = \sum_{i=k}^n \lambda_i$ and  $\mathcal{C}\sim kDPP(K)$.
\end{proposition}
%Up to a scale factor $\alpha$ 
The above bound  matches the expectation formula for DPPs for this specific $\alpha$.
Also, notice that it was remarked in \cite{CSSP} that if $\alpha = \sum_{i=k}^n \lambda_i$ then $\mathbb{E}_{\mathcal{C}\sim DPP(K/\alpha)}[|\mathcal{C}|] \leq k$.
%Secondly, the inequality \eqref{eq:UpperBound} is analogous to a bound obtained in (\cite{CSSP}, Lemma 5) for the Nystr\"om approximation error with respect to the nuclear norm.
The inequality \eqref{eq:UpperBound} is obtained thanks to the following Lemma with $l=k$.
\begin{lemma}[Eqn 1.3 in \cite{GuruswamiSinop}]
Let $\bm{\sigma}\in \mathbb{R}^n$ be a vector with entries $\sigma_1\geq \dots\geq \sigma_n \geq 0$. Let $k$ and $l$ be integers such that $k\geq l>0$. Then, we have
$
\frac{e_{k+1}(\bm{\sigma})}{e_{k}(\bm{\sigma})}\leq \frac{1}{k-l +1}\sum_{i=l+1}^{n}\sigma_i.
$\label{lemma:Guruswami}
\end{lemma}
With the help of Lemma \ref{lemma:Guruswami}, we can prove \eqref{eq:UpperBound}. \begin{proof}[Proof of Proposition \ref{prop:UpperBound}] Let $k\geq 1$. 
We can lower bound the ratio $\frac{e_{k-1}(\bm{\lambda}_{\hat{\ell}})}{e_k(\bm{\lambda}_{\hat{\ell}})}$ in \eqref{eq:kDPPreg} by using Lemma \ref{lemma:Guruswami}. Namely let $\bm{\sigma}$ be the vector $\bm{\lambda}_{\hat{\ell}}\in\mathbb{R}^{n-1}$ with entries sorted in decreasing order, and let $l=k$. Then, it holds that
$    \frac{e_{k}(\bm{\sigma})}{e_{k-1}(\bm{\sigma})}\leq \sum_{i=k}^{n-1}\sigma_i$. By using the definition of $\bm{\sigma}$, we find that, if $k<\ell$, we have $\sum_{i=k}^{n-1}\sigma_i =-\lambda_\ell + \sum_{i=k}^{n}\lambda_i $. Otherwise, if $k\geq \ell$, we have $\sum_{i=k}^{n-1}\sigma_i = \sum_{i=k+1}^{n}\lambda_i$. Hence, we find the upper bound
\[
\frac{e_{k}(\bm{\sigma})}{e_{k-1}(\bm{\sigma})}\leq \sum_{i=k}^{n-1}\sigma_i\leq \sum_{i=k}^{n}\lambda_i = \alpha,
\]
since $\bm{\lambda}\geq 0$.
%Then, for $\ell\in [n]$, this implies that
%$$
%\left(\frac{e_{k-1}(\bm{\lambda}_{\hat{\ell}})}{e_k(\bm{\lambda}_{\hat{\ell}})}  \lambda_\ell +1\right)^{-1}\leq \left(\frac{1}{\alpha}  \lambda_\ell +1\right)^{-1}.
%$$
Finally, the statement is proved by using the latter inequality and the identity \eqref{eq:kDPPreg}.
\end{proof}

%Interestingly, the proof of Lemma \ref{lemma:Guruswami} relies on the following Theorem.
%\begin{theorem}[Theorem 3.1 in \cite{GuruswamiSinop}]\label{Thm:Schur}
%For any $\bm{\lambda}\in \mathbb{R}^n\geq 0$, the ratio $e_{k+1}(\bm{\lambda})/e_{k}(\bm{\lambda})$ is Schur-concave.
%\end{theorem}
%In the sequel, we the majorization as $\bm{\rho}\succ \bm{\rho}'$ if \dots
%Thanks to Theorem \ref{Thm:Schur}, we know that if $\bm{\rho}\succ \bm{\rho}'$, then 
%\begin{equation}
%    \frac{e_{k+1}(\bm{\rho})}{e_{k}(\bm{\rho})}\leq %\frac{e_{k+1}(\bm{\rho}')}{e_{k}(\bm{\rho}')}
%\end{equation}
%Idea: check how we can construct the lower bound from that.
%\subsection{Ensemble of interpolators}
\begin{remark}[Upper bound]
Consider the term $\ell=n$ in \eqref{eq:kDPPreg}. Then, the additional term at the denominator can be lower bounded as follows:
\[
\frac{e_k(\bm{\lambda}_{\hat{n}})}{e_{k-1}(\bm{\lambda}_{\hat{n}})}\geq \frac{n-k}{k}\lambda_{n-1} \left(\frac{\lambda_{n-1}}{\lambda_{1}}\right)^{k-1}\geq 0,
\]
where we used that $e_k(\bm{\lambda}_{\hat{n}})$ includes $\binom{n-1}{k}$ terms. This bound is pessimistic although it instructs that a small $k$ benefits to the regularization.
\end{remark}
As we have observed, the formulae of Theorem \ref{thm:implicit} or Corollary \ref{Corollary:spectral} show that the expectation over diverse subsets implicitly regularize the inverse of the kernel matrix. The improvement of this bound is worth further investigation.
%In particular, the in-sample prediction is 
%\[
%\hat{y}(C) = K C K_{\mathcal{C}\mathcal{C}}^{-1} C^\top y,
%\]
%for each individual regressor, while the in-sample prediction of KRR is 
%\[
%\hat{y}_0 = K (K+\alpha\mathbb{I})^{-1}  y.
%\]
A related work \cite{DerezinskiAISTATS} uses the same formula given in Theorem \ref{thm:implicit} to study the convergence of a random block coordinate optimization method for Kernel Ridge Regression, but does not study the ridgeless limit.

\section{Experimental results}

Sampling a more diverse subset improves the performance of Nystr\"om approximation and KRR~\cite{fanuel2020diversity}. In these experiments, we discuss ensemble approaches for the ridgeless case. The following datasets\footnote{\url{https://archive.ics.uci.edu/ml/index.php}} are used:  %\texttt{Parkinson},
\texttt{Adult}, \texttt{Abalone}, \texttt{Wine Quality}, \texttt{Bike Sharing} and \texttt{CASP}. %Along~\cite{fanuel2020diversity}, 
We use 3 sampling algorithms with increasing diversity: uniform sampling, exact ridge leverage score sampling (RLS)~\cite{ElAlaouiMahoney} and kDPP sampling~\cite{kulesza2011k}. For larger datasets the BLESS algorithm~\cite{rudi2018fast} is used instead of RLS and DPP-VFX~\cite{derezinski2019exact} to speed up the sampling of a kDPP.  These algorithms have a relativity small re-sampling cost that motivates their use for ensemble approaches. RLS can be seen as a cheaper proxy for DPP sampling as done in~\cite{CSSP}. The different parameters and sample sizes are given in the Supplementary Material.   A Gaussian kernel with bandwidth $\sigma$ is used after standardizing the data. All the simulations are repeated 10 times, the averaged is displayed and the errorbars show the $0.25$ and $0.75$ quantile.

\paragraph{Ensemble Nystr\"om.}
The accuracy of the approximation is evaluated by calculating $\|K-\hat{K}\|_F/\|K\|_F$ with the ensemble Nystr\"om approximation $\hat{K} = \frac{1}{m}\sum_{i=1}^m K{C}_i (K_{\mathcal{C}_i\mathcal{C}_i}+\varepsilon\mathbb{I}_k)^{-1}C_i^\top K$ with $\varepsilon = 10^{-12}$ for numerical stability. We illustrate the use of diverse ensembles  on Figure~\ref{fig:Nystrom}. Averaging multiple Nystr\"om approximations improves the accuracy. The gain is the most apparent for the more diverse sampling algorithms. Similarly to the experiments in~\cite{kumar2009ensemble}, we see that uniform sampling combined with equal mixture weights does not improve performance. This is not the case when using more sophisticated sampling algorithms.

\begin{figure}[h]
	\centering
        \includegraphics[width=0.8\linewidth]{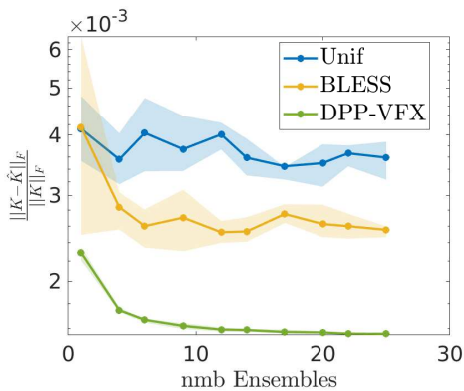}
	\caption{Ensemble Nystr\"om approximation on the Adult dataset. The relative Forbenius norm of the approximation is given in function of the number of ensembles.}
	\label{fig:Nystrom}
\end{figure}

\paragraph{Ensemble KRR.}
Following the implicit regularization of DPP samplings, we asses the performance of averaging ridgeless predictors trained on DPP subsets. Prediction is done by averaging the ridgeless predictors in an ensemble approach: $\bar{f}= \frac{1}{m} \sum_{i=1}^{m} f^\star_{\mathcal{C}_i}$.
We evaluate by the same procedure as in~\cite{fanuel2020diversity}. The dataset is split in $50\%$ training data and $50\%$ test data, so to make sure the train and test set have similar RLS distributions.   To evaluate the performance, the dataset is stratified, i.e., the test set is divided into 'bulk' and 'tail' as follows:
the bulk corresponds to test points where the RLS with regularization $\alpha = 10^{-4} \times n_{\mathrm{train}}$ are smaller than or equal to the 70\% quantile, while the tail of the data corresponds  to test points where the ridge leverage score is larger than the 70\% quantile.  This stratification of the dataset allows to visualize how the regressor performs in dense (small RLS) and sparser (large RLS) groups of the dataset. We calculate the symmetric mean absolute percentage error (SMAPE): $\frac{1}{n} \sum_{i=1}^{n} \frac{\left|y_{i}-\hat{y}_{i}\right|}{\left(\left|y_{i}\right|+\left|\hat{y}_{i}\right|\right) / 2}$ of each group. The results for exact sampling algorithms are visualised on Figure~\ref{fig:KKR:KKR}, approximate algorithms are given on Figure~\ref{fig:KKR:KKR_LS}. Combining multiple subsets shows a reduction in error. Following~\cite{fanuel2020diversity},  sampling a more diverse subset improves the performance of the KRR. Particularly diverse sampling has comparable performance for the bulk data, while performing much better in the tail of the data. Importantly, all the methods reach a stable performance before the number of points used by all interpolators exceeds the total number of training points.

\FloatBarrier

\subsection*{Acknowledgements}
\footnotesize{
EU: The research leading to these results has received funding from
the European Research Council under the European Union's Horizon
    2020 research and innovation program / ERC Advanced Grant E-DUALITY
    (787960). This paper reflects only the authors' views and the Union
    is not liable for any use that may be made of the contained information.
Research Council KUL: Optimization frameworks for deep kernel machines C14/18/068
Flemish Government:
FWO: projects: GOA4917N (Deep Restricted Kernel Machines:
        Methods and Foundations), PhD/Postdoc grant
Impulsfonds AI: VR 2019 2203 DOC.0318/1QUATER Kenniscentrum Data
        en Maatschappij
Ford KU Leuven Research Alliance Project KUL0076 (Stability analysis
    and performance improvement of deep reinforcement learning algorithms). The computational resources and services used in this work were provided by the VSC (Flemish Supercomputer Center), funded by the Research Foundation - Flanders (FWO) and the Flemish Government – department EWI.}

\bibliography{example_paper}
\bibliographystyle{alpha}

\appendix

\section{Parameters and dataset descriptions}
The parameters and datasets used in the simulations can be found in Table~\ref{Table:data}. The dataset dimensions are given by $n$ and $d$, $\sigma$ is the bandwidth of the Gaussian kernel, $k$ the size of the subset. The regularization parameter of the RLS is equal to $\lambda_{\mathrm{RLS}}$. The parameters for DPP-VFX correspond to $\bar{\mathrm{q}}_\mathrm{xdpp}$ and $\bar{\mathrm{q}}_\mathrm{bless}$. These are the oversampling parameters for internal Nystr\"om approximation of BLESS and DPP-VFX used to guarantee that everything terminates. Tuning parameters of the BLESS algorithm are  $q_0$, $c_0$, $c_1$, $c_2$.

\onecolumn

\begin{table*}[ht]
	\caption{Datasets and parameters used in the experiments.}
	\label{Table:data}
        \begin{center}
                \begin{tabular}{rccccccccccc}
                        \toprule
                        Dataset & $n$ & $d$  & $\sigma$ & $k$  & $\lambda_{\mathrm{RLS}}$ & $\bar{\mathrm{q}}_\mathrm{xdpp}$ & $\bar{\mathrm{q}}_\mathrm{bless}$ & $q_0$ & $c_0$& $c_1$& $c_2$\\ \midrule
%\texttt{Parkinson} & $5875$ & $20 $ & $5$ & $50$ & $10^{-4}$ & / & /  & / & / & / & /\\
\texttt{Adult} & $48842$ & $110 $ & $5$ & $250$ & $10^{-3}$ & $3$ & $3$  & $2$ & $2$ & $3$ & $3$\\
\texttt{Abalone}& $4177$ &$ 8 $ & $3$ & $50$ & $10^{-4}$ & / & /  & / & / & / & /\\
\texttt{Wine Quality} & $6497$ & $11$  & $5$ & $100$ & $10^{-4}$ & / & /  & / & / & / & /\\
\texttt{Bike Sharing} & $17389$ & $16 $ & $3$ & $250$  & $10^{-3}$ & $3$ & $3$  & $2$ & $2$ & $3$ & $3$\\ 			
%\texttt{Super} & $21263$ & $81$ & $10^{-3}$ & $3$ & $3$  & $2$ & $2$ & $3$ & $3$\\
\texttt{CASP} & $45730$ & $9$   & $2$ & $250$ & $10^{-3}$ & $3$ & $3$  & $2$ & $2$ & $3$ & $3$\\
 \bottomrule
 		\end{tabular}
 	\end{center}
 \end{table*}

\end{document}